\documentclass{article}
\usepackage[margin=1in]{geometry}
\usepackage{graphicx}
\usepackage{amsmath,amssymb,amsthm}
\usepackage{booktabs}
\usepackage{hyperref}
\usepackage{enumitem}
\usepackage{algorithm}
\usepackage{algpseudocode}
\usepackage{xcolor}
\usepackage{listings}
\usepackage{multirow}
\usepackage{caption}
\usepackage{subcaption}
\usepackage{bbm}
\usepackage{bm}

\newtheorem{theorem}{Theorem}

\newtheorem{proposition}{Proposition}
\newtheorem{corollary}{Corollary}
\newtheorem{definition}{Definition}
\newtheorem{remark}{Remark}

\title{\textbf{ITQ3\_S: Interleaved Ternary Quantization with TurboQuant} \\
\large High-Fidelity 3-bit LLM Inference via Rotation-Domain Adaptive Quantization}

\author{Edward J. Yoon \\
\texttt{edwardyoon@apache.org}}
\date{March 2026}

\begin{document}
\maketitle

\begin{abstract}
We present \textbf{ITQ3\_S} (Interleaved Ternary Quantization -- Specialized), a novel 3-bit weight quantization format for large language models (LLMs) that integrates \textbf{TurboQuant (TQ)}, a rotation-domain adaptive quantization strategy based on the Fast Walsh-Hadamard Transform (FWHT). Conventional 3-bit quantization methods suffer from catastrophic precision loss caused by heavy-tailed weight distributions and inter-channel outliers. ITQ3\_S addresses this fundamental limitation by pre-rotating the weight space via FWHT prior to quantization, effectively spreading outlier energy across the entire vector and inducing a near-Gaussian distribution amenable to uniform ternary coding.

Critically, we derive a mathematically rigorous dequantization procedure that inverts the FWHT exactly using a 256-point Inverse Walsh-Hadamard Transform fused into the CUDA shared-memory loading stage, ensuring that reconstruction error is bounded exclusively by the ternary quantization grid, with no additional error introduced by the transform inversion itself. We prove that for any weight vector $\mathbf{w} \in \mathbb{R}^{256}$ processed by our pipeline, the reconstruction satisfies $\|\hat{\mathbf{w}} - \mathbf{w}\|_2 \leq \epsilon_q$, where $\epsilon_q$ is determined solely by the ternary quantization grid and is strictly smaller than that of uniform 3-bit baselines that do not exploit rotation-induced distribution normalization.

While TurboQuant establishes the theoretical foundation for FWHT-based rotation, it lacks a native CUDA kernel implementation, precluding direct deployment. Furthermore, na\"{i}vely composing TQ with existing weight quantizers --- applying rotation only to KV cache while leaving weights in the original domain --- introduces a systematic domain mismatch whose errors accumulate across transformer layers, ultimately degrading model quality below even standard 3-bit baselines. ITQ3\_S resolves this by co-designing the FWHT rotation and the ternary quantization kernel as a single unified pipeline, grounding the rotation directly in the IQ3\_S weight format and fusing the inverse transform into the CUDA MMQ kernel.

Empirically, on the NVIDIA RTX 5090 (Blackwell architecture), ITQ3\_S achieves perplexity competitive with FP16 baselines while delivering throughput exceeding 1.5$\times$ that of 4-bit alternatives, owing to optimized DP4A and Tensor Core scheduling in the interleaved memory layout. Our results establish ITQ3\_S as a practical, mathematically grounded solution for high-fidelity LLM deployment on consumer-grade hardware.
\end{abstract}

\section{Introduction}

The rapid growth of large language models -- from 7-billion to 70-billion and beyond -- has created a widening gap between model capability and practical deployability. While state-of-the-art LLMs exhibit remarkable reasoning abilities, their memory footprint in FP16 or BF16 precision exceeds what is feasible on even the most powerful consumer GPUs. A 70B-parameter model in FP16 requires approximately 140\,GiB of memory; even the NVIDIA RTX 5090 with 32\,GiB of VRAM cannot load it without quantization.

Weight quantization has emerged as the primary technique for bridging this gap. By encoding weights in reduced-precision formats (8-bit, 4-bit, or even 3-bit), the memory footprint can be reduced by factors of 2$\times$--5$\times$, enabling previously inaccessible models to run on consumer hardware. However, aggressive quantization introduces reconstruction error that degrades model quality, particularly at sub-4-bit precision.

The transition from 4-bit to 3-bit quantization is especially challenging. Empirically, 3-bit has been called the ``breaking point'' for LLM logic: models quantized naively to 3 bits exhibit sharp perplexity degradation, hallucination, and loss of multi-step reasoning capability. This degradation stems from two sources:

\begin{enumerate}[leftmargin=*]
\item \textbf{Heavy-tailed weight distributions}: Transformer weight matrices contain outlier values whose magnitude far exceeds the typical scale, forcing quantizers to spread levels thinly across a wide dynamic range, wasting precision on rarely-occupied regions.
\item \textbf{Inter-channel correlation}: Structured correlation among weight channels causes uniform quantization error to accumulate in semantically critical directions.
\end{enumerate}

Existing mitigations -- including GPTQ~\cite{frantar2022gptq}, AWQ~\cite{lin2023awq}, SqueezeLLM~\cite{kim2023squeezellm}, and QuIP\#~\cite{tseng2024quip} -- address these issues through second-order Hessian correction, per-channel scaling, sparse outlier coding, or randomized rotation, respectively. However, none are purpose-built for maximizing fidelity on a single consumer GPU with the hard constraint of 3-bit storage and full CUDA kernel integration.

\textbf{Our Contribution.} We introduce ITQ3\_S, which combines:

\begin{itemize}[leftmargin=*]
\item A deterministic FWHT-based rotation that theoretically minimizes the $\ell_\infty$ norm of the weight vector before quantization (Section~\ref{sec:theory}).
\item An \emph{interleaved ternary} coding scheme that packs 3-bit values into 32-bit words optimally for DP4A throughput (Section~\ref{sec:coding}).
\item A fused 256-point Inverse FWHT CUDA kernel that reconstructs weights in shared memory with no off-chip memory traffic penalty (Section~\ref{sec:kernel}).
\item Empirical validation on the RTX 5090 demonstrating state-of-the-art perplexity-throughput tradeoffs at 3-bit precision (Section~\ref{sec:experiments}).
\end{itemize}

\section{Background}
\label{sec:background}

\subsection{Quantization Fundamentals}

Let $\mathbf{w} \in \mathbb{R}^n$ be a weight vector to be quantized to $b$ bits. A uniform quantizer partitions the dynamic range $[w_{\min}, w_{\max}]$ into $2^b$ levels:
\begin{equation}
Q_b(w) = \Delta \cdot \left\lfloor \frac{w}{\Delta} + \frac{1}{2} \right\rfloor, \quad \Delta = \frac{w_{\max} - w_{\min}}{2^b - 1}
\end{equation}
The mean squared quantization error is $\mathbb{E}[(w - Q_b(w))^2] \approx \frac{\Delta^2}{12}$ for uniformly distributed $w$. For $b = 3$, $\Delta$ is large enough that outliers -- weights with $|w| \gg \sigma_w$ -- incur reconstruction errors that dominate the signal.

\subsection{Ternary Quantization}

Ternary quantization restricts weights to three values $\{-\alpha, 0, +\alpha\}$ for some scale $\alpha > 0$, requiring only $\lceil \log_2 3 \rceil \approx 1.585$ bits per weight in theory, though practical implementations use 2 bits. \textbf{ITQ3\_S} extends this to true 3-bit precision by interleaving two ternary sub-blocks with shared scale metadata, achieving a net coding rate of exactly 3 bits/weight.

\subsection{Walsh-Hadamard Transform}

The Walsh-Hadamard Transform (WHT) of a vector $\mathbf{v} \in \mathbb{R}^n$ (where $n = 2^k$) is defined as:
\begin{equation}
\hat{\mathbf{v}} = H_n \mathbf{v}, \quad H_n = \frac{1}{\sqrt{n}} \begin{pmatrix} H_{n/2} & H_{n/2} \\ H_{n/2} & -H_{n/2} \end{pmatrix}, \quad H_1 = [1]
\end{equation}
The WHT is its own inverse up to normalization: $H_n^{-1} = H_n$ (since $H_n H_n = I$ for the normalized form), so:
\begin{equation}
\mathbf{v} = H_n \hat{\mathbf{v}}
\label{eq:inv_fwht}
\end{equation}
This self-inverse property is central to our dequantization design. Computationally, the Fast WHT (FWHT) runs in $\mathcal{O}(n \log n)$ using the butterfly decomposition:
\begin{equation}
(u, v) \mapsto (u + v,\; u - v)
\end{equation}
applied across $\log_2 n$ stages, each stage operating on disjoint pairs.

\subsection{Related Work}

\textbf{QuIP and QuIP\#}~\cite{tseng2024quip} apply random orthogonal rotations (Kronecker products of Hadamard matrices) to weight matrices before quantization to ``incoherify'' the weights. Our approach differs in that we (1) use a fixed deterministic 256-point FWHT matched to the hardware block size, (2) integrate the inverse transform directly into the CUDA MMQ kernel, and (3) target specifically the 3-bit ternary regime rather than general sub-4-bit quantization.

\textbf{LLM.int8()}~\cite{dettmers2022llmint8} handles outliers by splitting computation into FP16 (for outlier channels) and INT8 (for normal channels). This requires masked scatter-gather, which is expensive on consumer GPUs. ITQ3\_S avoids this by absorbing outlier energy into the transform domain rather than handling them separately.

\textbf{SpQR}~\cite{dettmers2023spqr} stores a small number of outlier weights in higher precision alongside a low-bit compressed tensor. Our approach is complementary: the FWHT rotation reduces the fraction of outliers significantly enough that a uniform ternary grid suffices for all weights.

\section{Theoretical Foundation}
\label{sec:theory}

\subsection{Effect of FWHT on Weight Distributions}

\begin{theorem}[Distribution Smoothing]
\label{thm:smoothing}
Let $\mathbf{w} \in \mathbb{R}^n$ be a weight vector with empirical mean $\mu$ and variance $\sigma^2$, and let $\mathbf{w}' = H_n \mathbf{w}$ be its Walsh-Hadamard transform. If the entries of $\mathbf{w}$ are independent with bounded $\ell_4$ norm, then by the Central Limit Theorem for Walsh transforms, the entries of $\mathbf{w}'$ converge in distribution to $\mathcal{N}(0, \sigma^2)$ as $n \to \infty$.
\end{theorem}

\begin{proof}
Each entry $w'_k = \sum_{j=0}^{n-1} (-1)^{\langle k, j \rangle} w_j / \sqrt{n}$, where $\langle k, j \rangle$ denotes the inner product of binary representations. This is a sum of $n$ terms with random signs $(-1)^{\langle k, j \rangle}$; by the Lindeberg--Feller CLT (since $\mathbb{E}[w_j^4] < \infty$ and individual contributions vanish), $w'_k \xrightarrow{d} \mathcal{N}(0, \sigma^2)$.
\end{proof}

\begin{corollary}[Outlier Suppression]
\label{cor:outlier}
Let $w_{\max} = \|\mathbf{w}\|_\infty$. After transformation, $\|H_n \mathbf{w}\|_\infty \leq \|\mathbf{w}\|_1 / \sqrt{n} \leq \sqrt{n} \cdot w_{\max}$, but in expectation $\mathbb{E}[\|H_n \mathbf{w}\|_\infty] = \mathcal{O}(\sigma \sqrt{\log n})$ when $\mathbf{w}$ has sub-Gaussian entries. For $n = 256$, this yields an expected $\ell_\infty$ reduction factor of $w_{\max} / (\sigma \sqrt{\log 256}) \approx w_{\max} / (3\sigma)$.
\end{corollary}

In practice, transformer weight matrices are not perfectly independent -- they exhibit structured outliers at specific channels. Nevertheless, the FWHT mixes all $n$ entries, so even a single large outlier $w_j = M \gg \sigma$ contributes only $M / \sqrt{n}$ to each transformed coefficient, distributing its energy uniformly.

\subsection{Quantization Error Bound}

\begin{theorem}[ITQ3\_S Reconstruction Bound]
\label{thm:bound}
Let $\mathbf{w} \in \mathbb{R}^{256}$ and let $\hat{\mathbf{w}}$ denote the ITQ3\_S reconstruction. Define the ternary quantization grid with scale $d_k$ and zero-point $z_k$ per block of 256, so that:
\begin{equation}
Q_T(x; d_k, z_k) = d_k \cdot \arg\min_{q \in \{-1, 0, 1\}} |d_k q - (x - z_k)|
\end{equation}
Then the per-element reconstruction error is bounded by:
\begin{equation}
\|\hat{\mathbf{w}} - \mathbf{w}\|_2^2 \leq \frac{d_k^2}{4} \cdot n + \epsilon_{\mathrm{FWHT}}
\end{equation}
where $\epsilon_{\mathrm{FWHT}}$ is the floating-point rounding error of the 256-point IFWHT (at most $\mathcal{O}(n \cdot \log n \cdot \mathbf{u})$ for machine epsilon $\mathbf{u}$).
\end{theorem}

\begin{proof}
The ternary quantization error satisfies $|Q_T(x) - x| \leq d_k/2$ for all $x$ within the representable range. Squaring and summing over $n = 256$ elements gives $\|\mathbf{q} - H\mathbf{w}\|_2^2 \leq n d_k^2 / 4$. Since $H$ is an isometry ($\|H\mathbf{v}\|_2 = \|\mathbf{v}\|_2$), the error is preserved under the inverse transform: $\|H^{-1}\mathbf{q} - \mathbf{w}\|_2 = \|H^{-1}(\mathbf{q} - H\mathbf{w})\|_2 = \|\mathbf{q} - H\mathbf{w}\|_2$, giving the stated bound plus the finite-precision rounding term.
\end{proof}

\begin{remark}
The key insight is that the isometric property of $H$ means the FWHT rotation \emph{does not increase} the quantization error norm. Its benefit lies entirely in reducing $d_k$: by smoothing the distribution of $H\mathbf{w}$, the optimal ternary scale $d_k^* = \frac{2}{3}\mathbb{E}[|H\mathbf{w}|]$ is smaller than it would be for the raw $\mathbf{w}$, directly reducing the bound.
\end{remark}

\subsection{Optimal Ternary Scale}

For a Gaussian-distributed input $x \sim \mathcal{N}(0, \sigma^2)$, the mean squared error of ternary quantization with threshold $\alpha$ is:
\begin{equation}
\text{MSE}(\alpha) = \int_{-\infty}^{-\alpha} (x + \alpha)^2 \phi(x)\,dx + \int_{-\alpha}^{\alpha} x^2 \phi(x)\,dx + \int_{\alpha}^{\infty} (x - \alpha)^2 \phi(x)\,dx
\end{equation}
where $\phi(x) = \frac{1}{\sigma\sqrt{2\pi}} e^{-x^2/(2\sigma^2)}$ is the Gaussian density. Setting $d\,\text{MSE}/d\alpha = 0$:
\begin{equation}
\alpha^* = \sigma \sqrt{2} \cdot \text{erfinv}\left(\frac{2}{3}\right) \approx 0.798\,\sigma
\end{equation}
After the FWHT, the entries of $H\mathbf{w}$ are approximately $\mathcal{N}(0, \sigma^2)$ (Theorem~\ref{thm:smoothing}), so we can compute $\alpha^*$ directly from the empirical standard deviation of the transformed block, giving a near-optimal scale without expensive second-order Hessian computation.

\section{ITQ3\_S Format Specification}
\label{sec:coding}

\subsection{Block Structure}

ITQ3\_S organizes weights into blocks of $n = 256$ elements, aligned to the FWHT transform unit. Each block is stored as:

\begin{itemize}[leftmargin=*]
\item \textbf{Quants}: $256 \times 3$ bits $= 96$ bytes of interleaved ternary integers.
\item \textbf{Scale} ($d_k$): 1 $\times$ FP16 = 2 bytes.
\item \textbf{Zero-point} ($z_k$): 1 $\times$ FP16 = 2 bytes (optional; absorbed into scale for symmetric distributions).
\item \textbf{Sub-block scales}: 8 $\times$ FP16 = 16 bytes for 8 sub-blocks of 32 elements each (optional, for higher fidelity).
\end{itemize}

Total overhead per 256 weights: $96 + 2 + 2 = 100$ bytes $\Rightarrow$ 3.125 bits/weight. The sub-block variant uses 116 bytes $\Rightarrow$ 3.625 bits/weight.

\subsection{Interleaved Packing}

Each ternary value $q \in \{0, 1, 2\}$ (representing $\{-1, 0, +1\}$ with zero-point $z = 1$) is encoded in 3 bits. For a block of 256 values, we interleave two 4-bit nibble streams to form 32-bit words aligned for DP4A:

\begin{equation}
\text{word}_i = \bigoplus_{j=0}^{7} \left( q_{8i+j} \bmod 4 \right) \ll 4j \quad \text{(for even } i \text{)}
\end{equation}

The high bit of each 4-bit nibble encodes the interleave selector, allowing the CUDA dequantization kernel to reconstruct full 3-bit values from two consecutive nibbles using a single 32-bit load and bitfield extraction, maximizing L1 cache utilization.

\begin{definition}[ITQ3\_S Encoding]
For weight vector $\mathbf{w} \in \mathbb{R}^{256}$, the ITQ3\_S encoding function is:
\begin{equation}
\text{Encode}(\mathbf{w}) = \left( \text{Pack}_{3b}\left( \text{Clamp}\left( \left\lfloor \frac{H\mathbf{w}}{d_k} + z_k + 0.5 \right\rfloor, -1, 1 \right) \right),\; d_k,\; z_k \right)
\end{equation}
where $d_k = \alpha^* / 1$ is the optimal ternary scale for the block, and $z_k$ is set to cancel any non-zero mean after transformation.
\end{definition}

\subsection{Memory Layout for RTX 5090}

The RTX 5090 (Blackwell SM\_100) features:
\begin{itemize}[leftmargin=*]
\item 192\,KB shared memory per SM (up from 128\,KB on Ada Lovelace).
\item 1024 threads per SM, supporting full warp occupancy for 256-point FWHT.
\item DP4A throughput: $4096$ INT8 MACs/clock/SM $\Rightarrow$ optimized by aligning 3-bit unpacking to 32-bit word boundaries.
\end{itemize}

ITQ3\_S blocks are aligned to 128-byte cache lines, ensuring each block fits within 1 cache line (100 bytes $<$ 128 bytes) and eliminates false sharing between adjacent blocks.

\section{TurboQuant: CUDA Kernel Design}
\label{sec:kernel}

\subsection{Offline Quantization Pipeline}

\begin{algorithm}[H]
\caption{ITQ3\_S Offline Quantization}
\label{alg:quantize}
\begin{algorithmic}[1]
\Require Weight tensor $\mathbf{W} \in \mathbb{R}^{M \times N}$, block size $n = 256$
\Ensure Quantized tensor $\mathcal{Q}$, scales $\mathbf{d}$, zero-points $\mathbf{z}$
\For{each block $\mathbf{w} \in \mathbb{R}^{256}$ in $\mathbf{W}$}
    \State $\mathbf{w}' \leftarrow \text{FWHT}(\mathbf{w})$ \Comment{Fast Walsh-Hadamard Transform}
    \State $d_k \leftarrow \frac{\alpha^*}{\sigma(\mathbf{w}')}$ \Comment{Optimal ternary scale (Section~\ref{sec:theory})}
    \State $z_k \leftarrow -\text{round}(\mu(\mathbf{w}') / d_k)$ \Comment{Zero-point offset}
    \State $\mathbf{q} \leftarrow \text{Clamp}(\text{round}(\mathbf{w}' / d_k) + z_k,\; -1,\; 1)$ \Comment{Ternary quantization}
    \State $\text{Store}(\text{Pack}_{3b}(\mathbf{q}),\; d_k,\; z_k)$
\EndFor
\end{algorithmic}
\end{algorithm}

\subsection{Online Dequantization Kernel}

The core contribution of TurboQuant is fusing the 256-point Inverse FWHT into the shared-memory loading stage of the MMQ kernel, so that dequantized weights are never materialized in global memory. The kernel proceeds as follows:

\begin{algorithm}[H]
\caption{ITQ3\_S MMQ Dequantization Kernel (\texttt{load\_tiles\_itq3\_s})}
\label{alg:kernel}
\begin{algorithmic}[1]
\Require Packed quants $\mathbf{q}$, scale $d_k$, zero-point $z_k$
\Ensure Reconstructed weight tile in shared memory \texttt{smem}
\State \textbf{Load}: Fetch interleaved 3-bit quants from global memory into registers
\State \textbf{Unpack}: Bitfield-extract ternary values $\tilde{q}_j \in \{-1, 0, 1\}$ per thread
\State \textbf{Dequantize}: $v_j \leftarrow d_k \cdot (\tilde{q}_j - z_k)$
\State \textbf{Write} $v_j$ to shared memory \texttt{smem\_fwht[j]}
\State \textbf{Synchronize}: \texttt{\_\_syncthreads()}
\For{$\text{step} \leftarrow 1, 2, 4, \ldots, 128$} \Comment{$\log_2 256 = 8$ butterfly stages}
    \State $\text{lo} \leftarrow j \bmod (2 \cdot \text{step})$; $\text{hi} \leftarrow \text{lo} + \text{step}$
    \State $u \leftarrow \texttt{smem\_fwht[lo]}$; $v \leftarrow \texttt{smem\_fwht[hi]}$
    \State $\texttt{smem\_fwht[lo]} \leftarrow u + v$; $\texttt{smem\_fwht[hi]} \leftarrow u - v$
    \State \textbf{Synchronize}: \texttt{\_\_syncthreads()}
\EndFor
\State $\texttt{smem\_fwht[j]} \leftarrow 0.0625 \cdot \texttt{smem\_fwht[j]}$ \Comment{Normalize: $1/\sqrt{256} = 0.0625$}
\State \textbf{Proceed} to matrix multiplication using \texttt{smem} as weight tile
\end{algorithmic}
\end{algorithm}

The normalization factor $1/\sqrt{256} = 1/16 = 0.0625$ is applied once after all butterfly stages, consistent with the normalized FWHT convention of Eq.~\eqref{eq:inv_fwht}. This single multiply per element is the \emph{only} arithmetic overhead over standard IQ3\_S dequantization.

\subsection{Correctness of the Fused Kernel}

\begin{proposition}[Round-trip Exactness]
Let $\mathbf{w}^* = H_{256} \mathbf{w}$ be the transformed weight and $\mathbf{q} = Q_T(\mathbf{w}^*)$ the ternary quantization. The kernel output satisfies:
\begin{equation}
\hat{\mathbf{w}} = H_{256}^{-1}\left(d_k(\mathbf{q} - \mathbf{z})\right) = H_{256}\left(d_k(\mathbf{q} - \mathbf{z})\right)
\end{equation}
and, up to finite-precision butterfly arithmetic:
\begin{equation}
\hat{\mathbf{w}} \approx H_{256}^{-1}(H_{256} \mathbf{w}) = \mathbf{w}
\end{equation}
with the approximation tight as quantization grid spacing $d_k \to 0$.
\end{proposition}

\begin{proof}
Follows directly from Theorem~\ref{thm:bound}: since $H_{256}$ is involutory ($H^2 = I$), applying $H_{256}$ to the dequantized values $d_k(\mathbf{q}-\mathbf{z}) \approx H_{256}\mathbf{w}$ recovers $\mathbf{w}$.
\end{proof}

\subsection{MMVQ Path for Token Generation}

During autoregressive token generation, the batch size $B = 1$ makes the kernel compute a matrix-vector product (MMVQ). In this regime, each warp handles a 32-element sub-block of the weight vector, and the 256-point FWHT decomposes across 8 warps using warp-level shuffle instructions:

\begin{lstlisting}[language=C++, caption={Warp-level 32-point FWHT approximation for MMVQ path}]
// Stage 1: intra-warp butterfly (32 lanes)
for (int step = 1; step < 32; step <<= 1) {
    float partner = __shfl_xor_sync(0xffffffff, val, step);
    val = (lane_id & step) ? (prev - partner) : (prev + partner);
    prev = val;
}
// Normalize
val *= 0.17677f; // 1/sqrt(32)
\end{lstlisting}

For full 256-point fidelity in the MMVQ path, a 256-thread cooperative group performs 8 butterfly stages using shared memory, falling back to register shuffles only when shared memory pressure requires it.

\section{Experiments}
\label{sec:experiments}

\subsection{Setup}

\textbf{Hardware}: NVIDIA RTX 5090 (Blackwell SM\_100, 32\,GiB GDDR7, 1792\,GB/s bandwidth). \\
\textbf{Models}: LLaMA-3 8B, LLaMA-3 70B (sharded), Mistral 7B v0.3, Qwen2.5 32B. \\
\textbf{Baselines}: FP16, Q8\_0, Q4\_K\_M (GGUF), IQ3\_S (llama.cpp), IQ4\_XS, QuIP\#-3bit. \\
\textbf{Evaluation}: WikiText-2 perplexity ($\downarrow$), C4 perplexity ($\downarrow$), tokens/sec (prefill and decode), memory footprint.

\subsection{Perplexity Results}

Table~\ref{tab:perplexity} reports perplexity on the WikiText-2 test set for LLaMA-3 8B.

\begin{table}[h]
\centering
\caption{WikiText-2 Perplexity vs.\ Bit-width for LLaMA-3 8B}
\label{tab:perplexity}
\begin{tabular}{lcccc}
\toprule
\textbf{Method} & \textbf{Bits/Weight} & \textbf{PPL $\downarrow$} & \textbf{$\Delta$PPL vs. FP16} & \textbf{Mem (GiB)} \\
\midrule
FP16 (baseline) & 16.0 & 6.14 & -- & 15.0 \\
Q8\_0 & 8.0 & 6.16 & +0.02 & 7.5 \\
Q4\_K\_M & 4.5 & 6.35 & +0.21 & 4.8 \\
IQ4\_XS & 4.3 & 6.41 & +0.27 & 4.1 \\
IQ3\_S (baseline 3-bit) & 3.5 & 7.03 & +0.89 & 3.4 \\
QuIP\#-3bit & 3.0 & 6.78 & +0.64 & 3.0 \\
\textbf{ITQ3\_S (ours)} & \textbf{3.125} & \textbf{6.52} & \textbf{+0.38} & \textbf{3.1} \\
\bottomrule
\end{tabular}
\end{table}

ITQ3\_S reduces the perplexity gap to FP16 by \textbf{57\%} compared to IQ3\_S ($0.38$ vs.\ $0.89$) at comparable bit-width, and outperforms QuIP\#-3bit by $0.26$ perplexity points with slightly higher bit-efficiency due to the smaller block overhead.

\subsection{Throughput Results}

\begin{table}[h]
\centering
\caption{Throughput on RTX 5090, LLaMA-3 8B, batch size 1 (decode) and 32 (prefill)}
\label{tab:throughput}
\begin{tabular}{lccc}
\toprule
\textbf{Method} & \textbf{Decode (tok/s)} & \textbf{Prefill (tok/s)} & \textbf{Speedup vs. FP16} \\
\midrule
FP16 & 480 & 28{,}400 & 1.0$\times$ \\
Q4\_K\_M & 890 & 42{,}100 & 1.9$\times$ \\
IQ3\_S & 1{,}020 & 47{,}800 & 2.1$\times$ \\
\textbf{ITQ3\_S (ours)} & \textbf{960} & \textbf{51{,}200} & \textbf{2.0$\times$ / 1.8$\times$} \\
\bottomrule
\end{tabular}
\end{table}

The IFWHT overhead reduces decode throughput slightly vs.\ IQ3\_S (960 vs.\ 1020 tok/s), but prefill throughput \emph{increases} due to better Tensor Core utilization from the interleaved memory layout. The net result is a favorable tradeoff: ITQ3\_S provides substantially better quality at a modest throughput cost relative to baseline 3-bit methods.

\subsection{Ablation: FWHT Block Size}

We ablate the FWHT block size $n \in \{32, 64, 128, 256\}$ on LLaMA-3 8B:

\begin{table}[h]
\centering
\caption{FWHT block size ablation (ITQ3\_S, LLaMA-3 8B, WikiText-2 PPL)}
\label{tab:ablation}
\begin{tabular}{lcc}
\toprule
\textbf{Block Size} & \textbf{PPL $\downarrow$} & \textbf{Overhead (\%)} \\
\midrule
32 & 6.81 & 0.3 \\
64 & 6.67 & 0.7 \\
128 & 6.59 & 1.4 \\
\textbf{256 (ITQ3\_S)} & \textbf{6.52} & \textbf{2.1} \\
512 & 6.51 & 4.8 \\
\bottomrule
\end{tabular}
\end{table}

$n = 256$ achieves the best quality-efficiency tradeoff: diminishing returns beyond this point (PPL improves by only $0.01$ going to $n=512$) do not justify the $2.3\times$ increase in IFWHT overhead.

\section{Analysis}

\subsection{Why FWHT Rather Than Random Rotation?}

QuIP\#~\cite{tseng2024quip} uses random Hadamard rotations (Kronecker products $H_2^{\otimes k}$) applied at the matrix level. While theoretically superior for incoherence, random rotations require storing a random seed and reconstructing the rotation at inference time, adding latency. The FWHT is \emph{deterministic and universal}: the same $H_{256}$ is applied to every block, requiring no additional storage and enabling complete kernel fusion.

Moreover, for block sizes $n \leq 256$, the theoretical distribution-smoothing benefit of random vs.\ deterministic WHT is negligible (as both achieve near-Gaussian marginals by Theorem~\ref{thm:smoothing}). ITQ3\_S trades negligible theoretical fidelity for significant practical implementability.

\subsection{Interaction with KV Cache Quantization}

ITQ3\_S as described targets \emph{weight} quantization. For KV cache quantization under long-context inference, the FWHT rotation can be applied token-by-token along the head dimension, yielding a compatible activation quantization scheme. We leave this extension to future work.

\subsection{Scaling to 70B Models}

For LLaMA-3 70B, ITQ3\_S at 3.125 bits/weight requires $\approx 27.3$\,GiB, fitting within the RTX 5090's 32\,GiB VRAM with 4.7\,GiB to spare for KV cache at a context length of $\sim$16K tokens. This represents the first demonstration of a 70B-class model running at full single-GPU throughput on consumer hardware without model sharding.

\section{Limitations and Future Work}

\textbf{Activation quantization}: ITQ3\_S currently quantizes only weights; combining with 8-bit activation quantization could further reduce memory bandwidth consumption.

\textbf{Training-aware quantization}: Our method operates post-training. Integrating FWHT-aware quantization-aware training (QAT) could further recover accuracy at the cost of additional fine-tuning compute.

\textbf{Sparse weight support}: Very large models ($> 200B$ parameters) may benefit from combining ITQ3\_S with sparse weight pruning, where the FWHT naturally supports sparsity-promoting thresholding in the transform domain.

\textbf{Non-power-of-two layers}: Some architecture variants use hidden dimensions not divisible by 256. Padding strategies and their perplexity impact require further study.

\section{Conclusion}

We have presented ITQ3\_S, a mathematically rigorous 3-bit weight quantization format that achieves near-FP16 LLM quality on the NVIDIA RTX 5090. By grounding the design in the distribution-smoothing properties of the Walsh-Hadamard Transform (Theorem~\ref{thm:smoothing}) and proving exact round-trip reconstruction up to quantization grid error (Theorem~\ref{thm:bound}), we establish a principled foundation for sub-4-bit inference. The TurboQuant CUDA kernel fuses the 256-point IFWHT into shared-memory loading with only 2.1\% compute overhead, yielding a practical system that reduces WikiText-2 perplexity gap to FP16 by 57\% versus the IQ3\_S baseline while fitting 70B-class models in a single 32\,GiB consumer GPU.

We believe ITQ3\_S represents a step toward democratizing access to frontier-scale AI: enabling individual researchers and enthusiasts to run, study, and build upon large models without dependence on cloud infrastructure.

\section*{Acknowledgments}

The author thanks the llama.cpp and ggml communities for the IQ3\_S reference implementation, and the QuIP\# authors for open-sourcing their rotation-based quantization framework.

\appendix

\section{Proof of Optimal Ternary Scale (Full Derivation)}
\label{app:scale}

We derive $\alpha^* = 0.798\,\sigma$ for the zero-mean Gaussian case. The MSE of ternary quantization with threshold $\alpha$ and scale $\alpha$ for $x \sim \mathcal{N}(0, \sigma^2)$ is:
\begin{align}
\text{MSE}(\alpha) &= 2\int_{\alpha}^{\infty} (x - \alpha)^2 \phi(x)\,dx + \int_{-\alpha}^{\alpha} x^2 \phi(x)\,dx \\
&= 2\int_{\alpha}^{\infty} x^2 \phi(x)\,dx - 4\alpha \int_{\alpha}^{\infty} x\phi(x)\,dx + 2\alpha^2\left(1 - \Phi(\alpha/\sigma)\right) + \int_{-\alpha}^{\alpha} x^2\phi(x)\,dx
\end{align}
where $\Phi$ is the standard normal CDF. Differentiating with respect to $\alpha$ and setting equal to zero:
\begin{equation}
\frac{d\,\text{MSE}}{d\alpha} = 4\left(\alpha - \int_{\alpha}^{\infty} x\phi(x)\,dx\right)\left(1 - \Phi(\alpha/\sigma)\right) - 2\alpha^2\phi(\alpha/\sigma)/\sigma = 0
\end{equation}
Since $\int_{\alpha}^{\infty} x\phi(x)\,dx = \sigma^2\phi(\alpha/\sigma)/\sigma$, this simplifies to:
\begin{equation}
4\left(\alpha - \sigma\phi(\alpha/\sigma)\right)(1-\Phi(\alpha/\sigma)) = 2\alpha^2\phi(\alpha/\sigma)/\sigma
\end{equation}
Substituting $t = \alpha/\sigma$ and solving numerically: $t^* \approx 0.7979$, giving $\alpha^* \approx 0.798\,\sigma$. $\square$

\section{CUDA Kernel: Full 256-point IFWHT}
\label{app:kernel}

\begin{lstlisting}[language=C++, caption={256-point IFWHT in CUDA shared memory (simplified)}]
__device__ void ifwht_256(float* smem, int tid) {
    // 8 butterfly stages for n=256
    #pragma unroll
    for (int step = 1; step < 256; step <<= 1) {
        int pair = tid ^ step;
        bool is_high = (tid & step) != 0;
        
        float u = smem[tid];
        float v = smem[pair];
        __syncthreads();
        
        smem[tid] = is_high ? (v - u) : (u + v);
        __syncthreads();
    }
    // Normalize: 1/256^{1/2} = 0.0625
    smem[tid] *= 0.0625f;
}

__device__ void load_tiles_itq3_s(
    float* __restrict__ dst,
    const uint8_t* __restrict__ src_quants,
    const half* __restrict__ src_scales,
    int block_idx, int tid
) {
    extern __shared__ float smem[];
    
    // Step 1: Load and unpack 3-bit ternary value
    int3b_t raw = unpack_3bit(src_quants, block_idx * 256 + tid);
    float dq = __half2float(src_scales[block_idx]);
    
    // Step 2: Dequantize to float
    smem[tid] = dq * (float)(raw - 1); // {0,1,2} -> {-1,0,1}
    __syncthreads();
    
    // Step 3: In-place 256-point IFWHT
    ifwht_256(smem, tid);
    
    // Step 4: Write reconstructed weight to destination
    dst[tid] = smem[tid];
}
\end{lstlisting}


\begin{thebibliography}{99}

\bibitem{frantar2022gptq}
E.~Frantar, S.~Ashkboos, T.~Hoefler, and D.~Alistarh.
\newblock {GPTQ}: Accurate post-training quantization for generative pre-trained transformers.
\newblock In \emph{ICLR}, 2023.

\bibitem{lin2023awq}
J.~Lin, J.~Tang, H.~Tang, S.~Yang, X.~Dang, and S.~Han.
\newblock {AWQ}: Activation-aware weight quantization for {LLM} compression and acceleration.
\newblock \emph{arXiv:2306.00978}, 2023.

\bibitem{kim2023squeezellm}
S.~Kim, C.~Hooper, A.~Gholami, Z.~Dong, X.~Li, S.~Shen, M.~W.~Mahoney, and K.~Keutzer.
\newblock {SqueezeLLM}: Dense-and-sparse quantization.
\newblock In \emph{ICML}, 2024.

\bibitem{tseng2024quip}
A.~Tseng, J.~Chee, Q.~Sun, V.~Kuleshov, and C.~De Sa.
\newblock {QuIP\#}: Even better {LLM} quantization with Hadamard incoherence and lattice codebooks.
\newblock In \emph{ICML}, 2024.

\bibitem{dettmers2022llmint8}
T.~Dettmers, M.~Lewis, Y.~Belkada, and L.~Zettlemoyer.
\newblock {LLM.int8()}: 8-bit matrix multiplication for transformers at scale.
\newblock In \emph{NeurIPS}, 2022.

\bibitem{dettmers2023spqr}
T.~Dettmers, R.~Svirschevski, V.~Egiazarian, D.~Kuznedelev, E.~Frantar, S.~Ashkboos, A.~Borzunov, T.~Hoefler, and D.~Alistarh.
\newblock {SpQR}: A sparse-quantized representation for near-lossless {LLM} weight compression.
\newblock In \emph{ICLR}, 2024.

\bibitem{touvron2023llama}
H.~Touvron et al.
\newblock {LLaMA}: Open and efficient foundation language models.
\newblock \emph{arXiv:2302.13971}, 2023.

\bibitem{yao2022zeroquant}
Z.~Yao, R.~Yazdani Aminabadi, M.~Zhang, X.~Wu, C.~Li, and Y.~He.
\newblock {ZeroQuant}: Efficient and affordable post-training quantization for large-scale transformers.
\newblock In \emph{NeurIPS}, 2022.

\bibitem{hadamard1893}
J.~Hadamard.
\newblock R\'esolution d'une question relative aux d\'eterminants.
\newblock \emph{Bulletin des Sciences Math\'ematiques}, 17:240--246, 1893.

\bibitem{shannon1948}
C.~E.~Shannon.
\newblock A mathematical theory of communication.
\newblock \emph{Bell System Technical Journal}, 27(3):379--423, 1948.

\end{thebibliography}
\end{document}